\documentclass[preprint, 12pt]{elsarticle}

\usepackage[utf8]{inputenc}
\usepackage[T1]{fontenc}
\usepackage{lmodern}
\usepackage{microtype}
\usepackage{booktabs}
\usepackage{tabularx}
\usepackage{array}
\usepackage{xcolor}
\usepackage{enumitem}
\usepackage{amsmath,amssymb,amsthm}
\usepackage{url}
\usepackage{tikz}
\usetikzlibrary{arrows.meta, positioning, calc}
\usepackage{listings}
\usepackage{placeins}                       
\usepackage[hidelinks]{hyperref}            

\newcolumntype{Y}{>{\raggedright\arraybackslash}X}

\newtheorem{lemma}{Lemma}
\newtheorem{corollary}{Corollary}
\newcommand{\pcode}[1]{\texttt{\small #1}}

\newcommand{\Mep}{\ensuremath{M^{\mathrm{ep}}_{A}}}
\newcommand{\Msem}{\ensuremath{M^{\mathrm{sem}}_{A}}}

\journal{Information Sciences}

\begin{document}

\begin{frontmatter}

\title{Episodic-to-Semantic Consolidation \\
Without Identity Drift}

\author[hit-sw]{Xue~Qin}
\ead{qinxue@me.com}

\author[hit-cs]{Simin~Luan}
\ead{luansiminiot@gmail.com}

\author[soochow]{Cong~Yang\corref{cor1}}
\ead{cong.yang@suda.edu.cn}

\author[hit-cs]{Zhijun~Li\corref{cor1}}
\ead{lizhijun\_os@hit.edu.cn}

\cortext[cor1]{Corresponding authors.}

\affiliation[hit-sw]{organization={School of Software, Harbin Institute of Technology},
  city={Harbin}, country={China}}
\affiliation[hit-cs]{organization={School of Computer Science and Technology, Harbin Institute of Technology},
  city={Harbin}, country={China}}
\affiliation[soochow]{organization={School of Future Science and Engineering, Soochow University},
  city={Suzhou}, country={China}}

\begin{abstract}
Long-running adaptive intelligent agents face a structural tension
between knowledge consolidation and information integrity. Memory
consolidation is conventionally treated as an agent-changing
operation: a model is fine-tuned, a prompt rewritten, a policy
distilled, or a reflection appended to the context that governs
future behaviour. In regulated autonomic deployment this is a
liability because the agent operates under commitments and audit
contracts that bind to a specific, cryptographically certified
identity. We propose to treat consolidation not as a mutation of
the planner or the identity manifest, but as a deterministic
function $f: \Mep \to \Msem$ from the agent's episodic event log
$\Mep$ (an append-only record of past interactions) to a
separately addressable semantic knowledge layer $\Msem$
(queryable, consolidated facts derived from that log); the
identity hash does not read $\Msem$, so consolidation updates
knowledge without changing the agent's certified identity. We give a formal
account of the agent representation, prove identity invariance
through a structural lemma on the manifest's hash-input set,
specify a deterministic aggregation algorithm whose outputs are
auditable database rows with explicit confidence and supporting-
event provenance, and validate the construction with synthetic
experiments demonstrating per-field correctness, byte-equal
identity across consolidation passes, and a mean $79.8\%$
reduction in unproductive planner attempts (95\% BCa CI
$[78.0\%, 81.5\%]$ across $10$ seeds) against a calibrated
Bayesian-shrunk baseline. The construction is a knowledge-update discipline
for autonomic agents in which lessons accumulate as queryable
facts while the agent's certified identity remains byte-equal
across its operational lifetime, with an embodied service agent
as the running case study.
\end{abstract}

\begin{keyword}
Episodic Memory \sep Semantic Consolidation \sep
Identity Invariance \sep Cryptographic Provenance \sep
Long-Running Embodied Agents \sep Information~Integrity
\end{keyword}

\end{frontmatter}

\section{Introduction}\label{sec:intro}

Long-running adaptive intelligent agents are increasingly deployed
in regulated environments where two demands collide. From an
adaptive-systems standpoint, the agent must calibrate to its
lived experience: per-task confidence must update with empirical
outcomes, risk posture must respond to incident history, and
operator-facing parameters must track ground-truth task
difficulty. From an information-integrity standpoint, the same
agent must present a stable, cryptographically certifiable
identity to compliance auditors, cross-deployment migration
tooling, and the downstream cyber-physical infrastructure that
consumes the agent's outputs. Embodied service agents in
hospitals, manufacturing cells, and warehouse logistics are the
running case: each is operated under commitments and audit
contracts that bind to a specific certified identity. If
knowledge consolidation silently mutates the agent, every
consolidation pass becomes a re-certification event. Either the
certificate is updated continuously (which compliance paperwork
cannot sustain) or the consolidation is treated as drift (which
the operator cannot~tolerate).

The tension is the standard expression of the stability-plasticity
problem under autonomic governance, but with an extra constraint
absent from the classical formulation: the agent's identity is not
defined by its weights, its prompt, or its current context window;
it is defined by a cryptographic digest over a separately
maintained manifest. Knowledge consolidation that changes the
weights, the prompt, or the in-context summary therefore breaks
identity integrity even when behaviour improvement is the
intended outcome. The conventional CL toolbox (parameter-side
regularisers, gradient projection, replay, reflection-driven
prompt rewrites) all bind the learned signal to the artefact whose
hash defines identity. None of them, taken as-is, preserves the
cryptographic certificate while still letting the agent become
operationally smarter from lived experience.

\begin{figure*}[!t]
\centering
\resizebox{\linewidth}{!}{%
\definecolor{t10blue}{HTML}{1F77B4}
\definecolor{t10orange}{HTML}{FF7F0E}
\definecolor{t10green}{HTML}{2CA02C}
\definecolor{t10red}{HTML}{D62728}
\definecolor{t10purple}{HTML}{9467BD}
\definecolor{t10brown}{HTML}{8C564B}
\begin{tikzpicture}[
  font=\footnotesize,
  box/.style={draw, rounded corners=2pt, minimum height=8mm,
              minimum width=20mm, align=center, inner sep=4pt,
              fill=white, line width=0.4mm},
  state/.style={box, draw=t10blue, text=black},
  agent/.style={box, draw=t10purple, text=black},
  derived/.style={box, draw=t10green, text=black},
  drift/.style={box, draw=t10red, line width=0.5mm, text=black},
  planner/.style={box, draw=t10purple, text=black},
  flow/.style={-{Latex[length=2mm]}, line width=0.6pt},
  readonly/.style={-{Latex[length=2mm]}, dashed, line width=0.6pt},
  hash/.style={font=\scriptsize\ttfamily}
]

\begin{scope}[xshift=0cm]
  \node[anchor=south] at (3.2, 4.8) {\textbf{(a) Classical consolidation}};
  \node[anchor=south, font=\scriptsize\itshape, text=red!60!black]
       at (3.2, 4.55) {identity mutates with knowledge};

  \node[state]   (epA)  at (0.5, 2.8) {Episodic\\events $\Mep$};
  \node[box, draw=t10orange] (fA) at (3.2, 2.8) {fine-tune\\reflect / distill};
  \node[drift]   (agA') at (5.9, 2.8) {Agent$'$\\{\scriptsize\ttfamily $h' \neq h$}};
  \node[agent, dashed] (agA) at (5.9, 1.2) {Agent\\{\scriptsize\ttfamily $h$}};

  \draw[flow] (epA) -- (fA);
  \draw[flow] (fA) -- (agA');
  \draw[flow, dashed, gray] (agA) -- node[right, font=\scriptsize, text=gray]
        {\textit{replaced}} (agA');

  \node[hash] at (5.9, 0.35) {$h \to h'$};
  \node[font=\scriptsize, text=red!70!black]
       at (5.9, 0.05) {certificate re-event};
\end{scope}

\draw[gray!50, dashed, line width=0.4pt] (7.5, 0.0) -- (7.5, 4.8);

\begin{scope}[xshift=8cm]
  \node[anchor=south] at (3.2, 4.8) {\textbf{(b) Consolidation as derived layer}};
  \node[anchor=south, font=\scriptsize\itshape, text=green!40!black]
       at (3.2, 4.55) {identity byte-equal across updates};

  \node[agent]   (manB) at (0.5, 3.8) {Manifest $M$};
  \node[hash, anchor=west, text=green!30!black] at (1.6, 3.8) {$h=\mathtt{SHA256}(M)$};

  \node[state]   (epB)  at (0.5, 2.8) {Episodic\\events $\Mep$};
  \node[box, draw=t10orange] (fB) at (3.2, 2.8) {$f(\cdot)$\\deterministic};
  \node[derived] (semB) at (5.9, 2.8) {Semantic\\$\Msem$};
  \node[planner] (plB) at (5.9, 1.0) {Planner};

  \draw[flow] (epB) -- (fB);
  \draw[flow] (fB) -- (semB);
  \draw[readonly] (semB) -- node[right, font=\scriptsize] {read-only} (plB);

  \node[font=\scriptsize, anchor=west, text=green!40!black]
       at (0.0, 0.05) {$M$ unchanged $\Rightarrow h=h$};
\end{scope}

\end{tikzpicture}%
}
\caption{Two stances on memory consolidation in regulated embodied
deployment. \textbf{(a)} Classical consolidation binds the learned
output to the agent itself: fine-tuning, reflection, distillation,
or in-context-memory rewrites all change the artefact whose hash
defines the certified identity. Every consolidation pass becomes a
re-certification event. \textbf{(b)} The construction in this paper
separates the manifest $M$ (whose hash $h$ binds the certificate)
from a derived semantic layer $\Msem = f(\Mep)$. The hash inputs
exclude $\Msem$ by construction (Lemma~\ref{lem:invariance}); the
planner reads $\Msem$ through a read-only surface and never
mutates the manifest. The robot accumulates operational knowledge
without moving its identity certificate.}
\label{fig:hero}
\end{figure*}
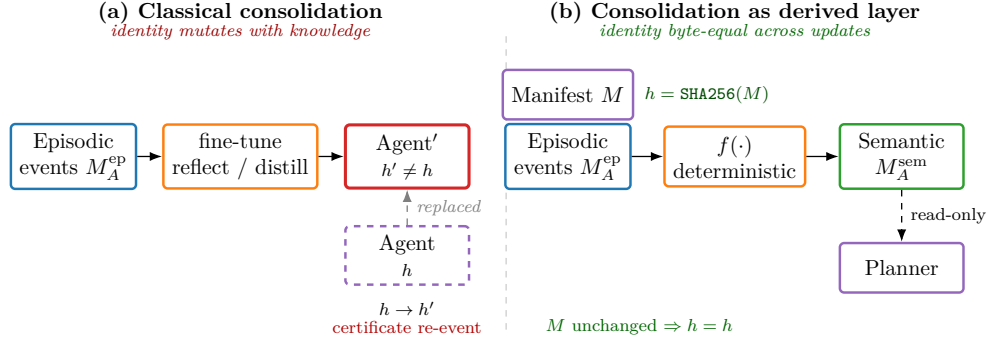

This paper asks: can an adaptive intelligent agent consolidate
operational knowledge from episodic experience while its identity
certificate remains byte-equal? Our answer, which we refer to as
\emph{identity-stable consolidation}, treats consolidation not as
a mutation of the planner or the identity manifest, but as a
deterministic function over episodic memory whose output is a
separately addressable semantic knowledge layer \Msem{} which
downstream planners query for grounding facts. The approach
factors memory into two stores (an episodic event log \Mep{}
and a derived semantic layer \Msem{}) and constrains
consolidation to write only the semantic store, leaving the
manifest's hash-input set untouched. The identity hash therefore
does not read this layer, so consolidation updates knowledge
without changing who the agent is in any sense the auditor, the
insurer, or the operator can~detect.

The setting is an autonomic-governance substrate for long-running
deployed agents. We assume an underlying architecture that
separates an immutable identity certificate (a manifest hashed
into a single fixed-length value) from the operational
scaffolding the agent uses moment-to-moment, and persists
episodic events under that identity in an append-only store
partitioned by identity hash. Such substrates are consistent with
prior work on identity-orthogonal memory layers for long-running
adaptive agents under autonomic governance. We treat the
existence of the manifest hash and the episodic store as
substrate; this paper's contribution is the knowledge-consolidation
layer above them, and the proof that this layer can update
without moving the cryptographic identity.

\paragraph{Motivating example} Suppose a robot's
\pcode{manipulation.grasp} capability fails on glass cups at
$15$~N three times and succeeds at $25$~N twelve times across one
week of operation. A robot that merely stores these events has
memory but no operational knowledge: each future grasp begins
from the same default. A robot that consolidates can answer a
future planner query with a fact such as ``for \pcode{glass\_cup}
under this environment, recommended grasp force is $25$~N with
confidence $0.83$ over $15$ observations.'' That fact
is not a policy rewrite, not a manifest mutation, and not a
fine-tuning gradient. It is a derived row in a semantic table,
computed from the episodic rows by a deterministic rule, and
recoverable byte-for-byte from the same inputs. A reviewer can
ask why the recommendation is $25$~N and walk back to the
supporting events. Figure~\ref{fig:worked-example} walks through
this same example end-to-end, including the manifest hash check
that completes once consolidation is done.

\begin{figure*}[!t]
  \centering
  \includegraphics[width=\linewidth]{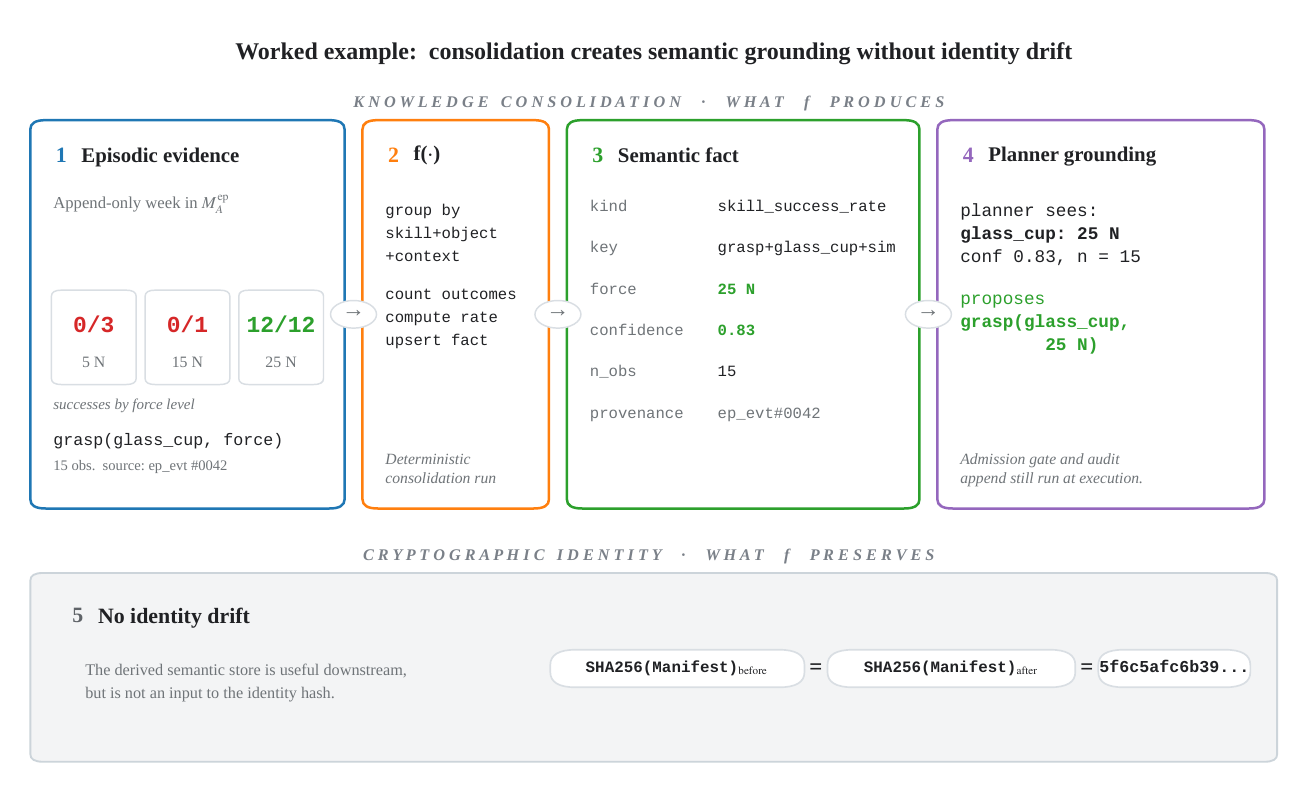}
  \caption{End-to-end worked example. The top row
    (\emph{knowledge consolidation}) walks through what $f(\cdot)$
    produces; the bottom band (\emph{cryptographic identity})
    states what $f(\cdot)$ preserves.
    \textbf{(1)} Episodic evidence aggregates a week of
    \pcode{grasp(glass\_cup, force)} attempts by force level:
    zero successes at $5$~N (three tries), zero at $15$~N
    (one try), and twelve successes at $25$~N out of twelve.
    \textbf{(2)} The deterministic consolidation function
    $f(\cdot)$ (\S\ref{sec:v1}) groups the rows by
    \pcode{skill+object+context}, counts outcomes, computes the
    success rate, and upserts a single semantic fact.
    \textbf{(3)} The output is one row in the
    \pcode{semantic\_facts} table, carrying the recommended force,
    a confidence value, the observation count, and a
    \pcode{provenance} pointer back to the supporting episodic event.
    \textbf{(4)} A downstream planner queries the read-only
    semantic endpoint, grounds its next intent in the consolidated
    fact, and emits a structured action that the runtime gates
    apply policy and audit on.
    \textbf{(5)} The manifest hash is recomputed before and after
    the consolidation pass and is byte-equal
    (Lemma~\ref{lem:invariance}): $M^{\mathrm{sem}}_A$ is not a
    hash input, so updating the semantic store cannot move the
    identity certificate.}
  \label{fig:worked-example}
\end{figure*}

\paragraph{Framework overview} Identity-stable consolidation is
realised through four elements that together constitute the
framework developed in the remainder of the paper.
\textbf{(i)} A formal model of agent state as a tuple
$(M, \Mep, \Msem, \pi)$ that explicitly separates the identity
manifest $M$ (the bytes hashed into the identity certificate)
from the operational substrate the agent uses moment to moment
(\S\ref{sec:formal}). \textbf{(ii)} A structural identity-invariance
result (Lemma~\ref{lem:invariance}) that excludes the semantic
layer \Msem{} from the manifest's hash-input set by construction
rather than by runtime assertion. \textbf{(iii)} A deterministic v1
consolidation algorithm $f: \Mep \to \Msem$ that materialises the
semantic store as auditable database rows with explicit confidence,
observation counts, and supporting-event provenance
(\S\ref{sec:v1}). \textbf{(iv)} A read-only planner interface
(\S\ref{sec:planner}) that exposes \Msem{} as grounding facts to
the downstream planner without giving the planner write access to
the manifest or the semantic store. The remainder of the paper
develops each element in turn and then reports empirical
validation (\S\ref{sec:eval}).

\paragraph{Contributions} The paper's contribution is formal,
architectural, and empirical, and is positioned as an
information-integrity discipline for adaptive intelligent agents
under autonomic governance.
\begin{enumerate}
\item \textbf{A structural identity-invariance result.} We define
semantic memory as $\Msem = f(\Mep)$ and prove, via input-set
inspection of the manifest's hash
(Lemma~\ref{lem:invariance}, Corollary~\ref{cor:iterated}), that
the cryptographic identity hash is byte-equal across arbitrary
sequences of consolidation passes, by construction rather than by
runtime check. The result follows from the type signature of
$f$ and the field set of the manifest dataclass, not from a
runtime assertion.
\item \textbf{A deterministic, auditable consolidation algorithm.}
We specify a v1 consolidation pass whose outputs are plain
database rows with explicit confidence, observation counters, and
supporting-event provenance, enabling end-to-end traceability from
any planner-consulted fact back to the supporting episodic
evidence. The pass is idempotent, order-invariant, and
recoverable byte-for-byte from the same inputs under the same
rule version.
\item \textbf{An empirical validation under three controls.} We
ablate the contribution of consolidation structure against
no-memory, uniform-confidence, and calibrated Bayesian-shrunk
controls on a 1000-decision synthetic benchmark across $10$
seeds (Table~\ref{tab:v3-controls}). The calibrated control
retires the straw-man critique of the no-memory baseline: a
mean $79.82\%$ reduction in unproductive planner attempts
(95\% BCa CI $[78.02\%, 81.49\%]$) holds against a held-out
Bayesian-shrunk confidence estimate, while the cryptographic
identity hash remains byte-equal across all consolidation passes
(verified through SHA-256 recomputation, see V2 fixture).
\end{enumerate}

An LLM-assisted v2 is sketched only as future work
(\S\ref{sec:v2}) because it introduces non-determinism that
complicates audit; deploying it before v1 stabilises would
inherit the drift problem we are trying to avoid. Planner
integration is realised through a read-only query interface
(\S\ref{sec:planner}) that exposes \Msem{} as grounding facts
without giving the planner write access to the semantic store or
the manifest.

\section{Related Work}\label{sec:related}

We position our contribution against five research currents, led
by autonomic-computing and adaptive-systems work that frames the
broader problem, and followed by the more specialised
sub-currents whose mechanisms our construction borrows from or
contrasts with.

\subsection{Autonomic Computing and Adaptive Intelligent Agents}
\label{sec:related:autonomic}

The autonomic-computing tradition frames the central problem of
long-running deployed systems as continuous self-adaptation:
agents monitor their operating context, analyse it against
declared objectives, plan corrective action, and execute the
plan, with knowledge accumulated across iterations of this
loop. The MAPE-K reference architecture
\citep{kephart-2003-autonomic} is the canonical realisation, and
the broader family of self-adaptive systems work has elaborated
it into a programme on architectural support for runtime
adaptation under explicit goals
\citep{garlan-2004-rainbow,cheng-2009-self-adaptive-roadmap}.
Within this tradition, the central methodological question is
how knowledge accumulated by the autonomic loop is represented,
updated, and shared between the loop's monitoring and execution
phases without compromising the system's externally visible
contracts.

Recent work in \emph{Information Sciences} has elaborated this
question across several axes that touch our construction
directly. \citet{wu-2023-emotional-belief} studies adaptive-belief
update in agent-based negotiation; \citet{tan-2024-q-learning-update}
develops heterogeneous update strategies for reinforcement-learning
agents; \citet{geng-2024-data-knowledge-belief-rule} combines
data-driven and knowledge-driven belief-rule learning;
\citet{zhang-2024-knowledge-comm} introduces knowledge-guided
coordination between cooperative multi-agent systems; and
\citet{wang-2025-fuzzy-knowledge} formalises fuzzy-knowledge
inference for dynamic task allocation. These works share an
implicit substrate assumption: the knowledge surface (be it a
belief table, a Q-function, or a belief-rule base) is also the
substrate whose update defines the agent. Our construction
parts from this assumption by separating the knowledge surface
that updates (semantic memory \Msem{}) from the substrate that
defines the agent's externally visible identity (the manifest
hashed into the identity certificate).

Modern realisations of the autonomic loop in the LLM-agent space
inherit its shape while differing in substrate: reflection-based
language agents such as \citet{reflexion} and
\citet{generative-agents} store verbal feedback and feed it
into later episodes; \citet{voyager} builds a persistent skill
library through trial; \citet{memgpt} manages a memory hierarchy
under context pressure; \citet{a-mem} treats memory itself as
agentic, with the agent deciding what to record and how to index
it. \citet{sumers-2023-coala} frames these designs as cognitive
architectures for language agents, mapping each onto memory,
decision-making, and learning loops. Across this body of work
the artefact that consolidates is also the artefact that is
deployed: the same parameter set, prompt context, or skill
library that produces behaviour is the one that learns. The
property is operationally effective but structurally
incompatible with an identity certificate that must remain
byte-equal across the adaptive loop. The present paper preserves
the autonomic loop's substantive content (continuous knowledge
update from lived experience) while inverting the bind:
consolidation produces a separately addressable artefact whose
update does not touch the agent's cryptographic identity. The
separation can be stated information-theoretically: by a
mutual-information argument
\citep{cover-2006-information-theory}, the identity hash,
computed only from the manifest's hash-input bytes, carries zero
information about the bytes of \Msem{}; therefore arbitrary
updates to \Msem{} cannot influence the hash. The formal account
in \S\ref{sec:formal} makes this argument~structural.

\subsection{Knowledge Representation and the Episodic-Semantic Distinction}
\label{sec:related:kr}

The vocabulary of knowledge representation and reasoning
\citep{brachman-2004-knowledge-representation} supplies the
formal substrate within which the episodic-semantic distinction
can be precisely stated. The distinction itself descends from
\citet{tulving-1972}: episodic memory holds traces of specific
events; semantic memory holds structured knowledge about the
world. \citet{squire-1992} argues that declarative memory in
mammals decomposes along the same axis, with the hippocampus
mediating event traces and the cortex holding consolidated
knowledge. Complementary learning systems theory
\citep{mcclelland-1995,kumaran-2016} develops this into an
explicit two-store model: a fast-learning hippocampal system
encodes individual episodes and a slower neocortical system
extracts statistical regularities, with sleep-time replay
\citep{wilson-mcnaughton-1994} as the canonical biological
mechanism by which knowledge moves between the two. These
traditions motivate the distinction between event logs and
extracted knowledge, and they motivate the design choice of
representing extracted knowledge in a separate substrate from
the event log itself. They do not address an engineered identity
certificate because the cognitive-science agent is the brain
and there is no separate ``identity hash'' to preserve. Our
construction inherits the two-store decomposition but re-anchors
it in an information-integrity setting: the identity-bearing
substrate is the manifest's hash, and the consolidation
function's correctness target is the byte-equality of that hash
rather than any neural-fidelity claim.

\subsection{Continual Learning and Replay-Based Consolidation}
\label{sec:related:cl}

Reinforcement learning uses replay buffers, offline datasets,
and policy updates to turn experience into better
behaviour~\citep{mnih-2015,levine-2020}. The continual-learning
literature has elaborated this into a family of mechanisms for
acquiring new knowledge without overwriting old: parameter-side
regularisers such as EWC \citep{kirkpatrick-2017-ewc} and
distillation-based variants such as LwF \citep{li-2018-lwf};
gradient-projection methods such as GEM
\citep{lopez-paz-2017-gem} and its averaged-gradient extension
A-GEM \citep{chaudhry-2019-agem}; and replay-based baselines
such as CLEAR \citep{rolnick-2019-clear} and DER
\citep{buzzega-2020-der}. \citet{vandeven-2022} surveys the
incremental-learning design space; \citet{parisi-2019} reviews
lifelong learning more broadly. Across all of this work the
agent is normally identified with its learned policy. A
consolidation step that updates policy parameters changes the
agent in the relevant sense. The property is acceptable when the
research object is performance improvement; it is less
acceptable when an autonomic deployment requires stable identity
across time. The underlying assumption that the agent \emph{is}
its weights does not hold in our setting, where a deployed
agent is the combination of a frozen identity manifest, a
swappable planner, an audited capability registry, and the
persistent memory layers studied here.

\subsection{Memory Architectures for Language and Embodied LLM Agents}
\label{sec:related:llm-agents}

A related current focuses specifically on how language-model
agents retain and reuse information across episodes.
\citet{reflexion} stores verbal feedback and feeds it into later
episodes; \citet{generative-agents} use memory retrieval and
reflection to produce higher-level summaries;
\citet{voyager} builds a persistent skill library;
\citet{memgpt} manages a memory hierarchy under context
pressure; \citet{zhong-2024-memorybank} attaches a forgetting
curve so older, less-rehearsed memories decay; \citet{a-mem}
treats memory itself as agentic, with the agent deciding what to
record and how to index it. For embodied agents,
\citet{saycan} grounds language-model intent in robotic
affordances by composing skills the runtime can actually execute.
These mechanisms all bind the consolidated output to the engine
that produced it: a different model, prompt, or policy after
consolidation is, operationally, a different agent. Our
construction inverts the bind: consolidation produces a
separately addressable artefact whose update does not change the
agent's certificate.

\subsection{Intention-Centred Agent Architectures}
\label{sec:related:bdi}

The BDI tradition \citep{bratman-1987,rao-georgeff-1995-bdi}
factors an agent's mental state into beliefs, desires, and
intentions, and pins observable behaviour to the intentions an
agent is currently committed to. The identity certificate
introduced here plays a structurally similar role: it freezes a
specific set of commitments (manifest, policy version, audited
capability set) and forces any update path to be visible to the
operator. The present paper differs from BDI by routing the
commitment through a hashed identity manifest rather than
through an explicit intention-revision protocol, and by treating
semantic memory as identity-orthogonal rather than as a
derivable view over current beliefs.

\subsection{Identity-Orthogonal Memory}
\label{sec:related:identity-orth}

Closest in spirit is a line of work on persistent episodic
memory that separates the identity manifest from the episodic
event log so that logging events does not alter the agent's
certificate. The present paper extends that programme by adding
a \emph{derived} memory layer (semantic facts) whose update is
likewise excluded from the identity hash by construction, and by
proving the exclusion structurally rather than asserting it as a
runtime invariant.

\section{Method}\label{sec:method}

This section develops identity-stable consolidation in four parts.
\S\ref{sec:formal} formalises the agent state, proves the
structural identity-invariance result
(Lemma~\ref{lem:invariance}), and gives a high-level walkthrough
of the four framework elements introduced in \S\ref{sec:intro}.
\S\ref{sec:v1} specifies the deterministic consolidation
algorithm, its determinism and idempotence properties, and the
audit chain that keeps every semantic row traceable to supporting
evidence. \S\ref{sec:v2} discusses an LLM-assisted extension as a
future direction. \S\ref{sec:planner} describes the read-only
planner integration interface.

\subsection{Framework}\label{sec:formal}

\subsubsection{Architecture Overview}\label{sec:formal-overview}

This section presents identity-stable consolidation as four
interacting elements before turning to the formal account that pins
down the identity-invariance guarantee. The four elements were
named informally in \S\ref{sec:intro} and visualised as the
derived-layer construction in Figure~\ref{fig:hero}(b); here we
describe how they fit together and where the remainder of the
paper develops each one.

\noindent\textbf{(i) Manifest, episodic store, and semantic store.}
The agent state separates an identity manifest~$M$ from two
operational substrates: an episodic event log $\Mep$ that
accumulates timestamped observations append-only, and a semantic
store $\Msem$ that materialises consolidated facts derived from
$\Mep$. The identity hash $h = \mathrm{SHA256}(M)$ binds the
agent's cryptographic certificate to a fixed-field manifest
dataclass, so growth in $\Mep$ or $\Msem$ has no effect on $h$ by
construction. Section~\ref{sec:formal-state} formalises the state
tuple and the canonical serialisation rules.

\noindent\textbf{(ii) Structural identity invariance.} The first
guarantee we prove is that the consolidation operation
$f : \Mep \to \Msem$ cannot move $h$. The proof is structural
rather than runtime: inspection of the manifest dataclass's field
set shows that $\Msem$ is not a hash input, so any $f$ writing only
to the semantic store leaves the canonical serialisation of $M$
unchanged (Lemma~\ref{lem:invariance},
Corollary~\ref{cor:iterated}). Section~\ref{sec:formal-invariance}
states the lemma and the input-set argument that supports it.

\noindent\textbf{(iii) Deterministic v1 consolidation.} The first
concrete realisation of $f$ is a deterministic SQL-style aggregation
pass that turns structured episodic payloads into semantic rows
with explicit confidence values, observation counts, and
supporting-event provenance. The pass is idempotent and
order-invariant; the same episodic state under the same rule
version produces a byte-equal semantic snapshot on every re-run.
Section~\ref{sec:v1} specifies the schema, the algorithm, and the
audit chain through which every semantic row remains traceable to
its supporting events. An LLM-assisted v2 is discussed as future
work (\S\ref{sec:v2}).

\noindent\textbf{(iv) Read-only planner interface.} A downstream
planner consumes $\Msem$ through a read-only HTTP endpoint that
surfaces top-ranked facts as grounding for the planner's prompt.
The interface deliberately does not give the planner write access
to $\Msem$ or $M$; semantic memory is an information surface, not
an authority surface. The interface specification, including
auditability of fact-consultation, is in
Section~\ref{sec:planner}.

With these four elements named, the remainder of this section
gives the formal state tuple (\S\ref{sec:formal-state}) and the
identity-invariance lemma (\S\ref{sec:formal-invariance}) on which
the rest of the paper builds.

\subsubsection{Agent State}\label{sec:formal-state}

We model the agent state as a tuple $(M, \Mep, \Msem, \pi)$ where
$M$ is the identity manifest (a frozen dataclass with a fixed,
finite set of fields), \Mep{} is the episodic event log (an
append-only sequence of timestamped events partitioned by
$\mathrm{SHA256}(M)$), \Msem{} is the derived semantic memory
table, and $\pi$ is the planner. The identity hash is
$h = \mathrm{SHA256}(M)$. The consolidation operation is
$\Msem{} \leftarrow f(\Mep{})$, a deterministic function.

\subsubsection{Structural Identity Invariance}\label{sec:formal-invariance}

The construction's correctness rests on a structural property of
the manifest's hash-input set rather than on a runtime check.
Concretely, we identify a class of operations that, by their
signature alone, cannot modify the manifest and therefore cannot
change the identity hash.

\begin{lemma}[Structural identity invariance]
\label{lem:invariance}
Let $\mathcal{F}_{\mathrm{sem}}$ denote the class of operations
$f : \Mep \to \Msem$ whose only write target is the semantic
store. For every $f \in \mathcal{F}_{\mathrm{sem}}$ and every
well-formed agent state $(M, \Mep, \Msem, \pi)$, applying $f$
leaves $M$ syntactically unchanged. Consequently, by determinism
of $\mathrm{SHA256}(\cdot)$ over canonical serialisations,
$h = \mathrm{SHA256}(M)$ is byte-equal before and after applying
any $f \in \mathcal{F}_{\mathrm{sem}}$.
\end{lemma}

\begin{proof}[Proof (by input-set inspection)]
The hash $h$ is a deterministic digest whose input is the
canonical serialisation of $M$'s fields. Inspection of the
manifest dataclass declared in the runtime confirms
that the field set is fixed at type-design time and does not
include any reference to $\Msem$. Operations in
$\mathcal{F}_{\mathrm{sem}}$ have $\Msem$ in their write footprint
and $M$ outside it; their effect on the agent state is therefore
disjoint from the hash-input set, and the canonical serialisation
of $M$ is unchanged. SHA-256 over the same byte string yields the
same digest. The result is structural: it follows from the type
signature of $f$ and the field set of $M$, not from a runtime
assertion. Figure~\ref{fig:hash-inputs} visualises the input set
and the structural exclusion that the proof rests on.
\end{proof}

\begin{figure}[!ht]
  \centering
  \includegraphics[width=\linewidth]{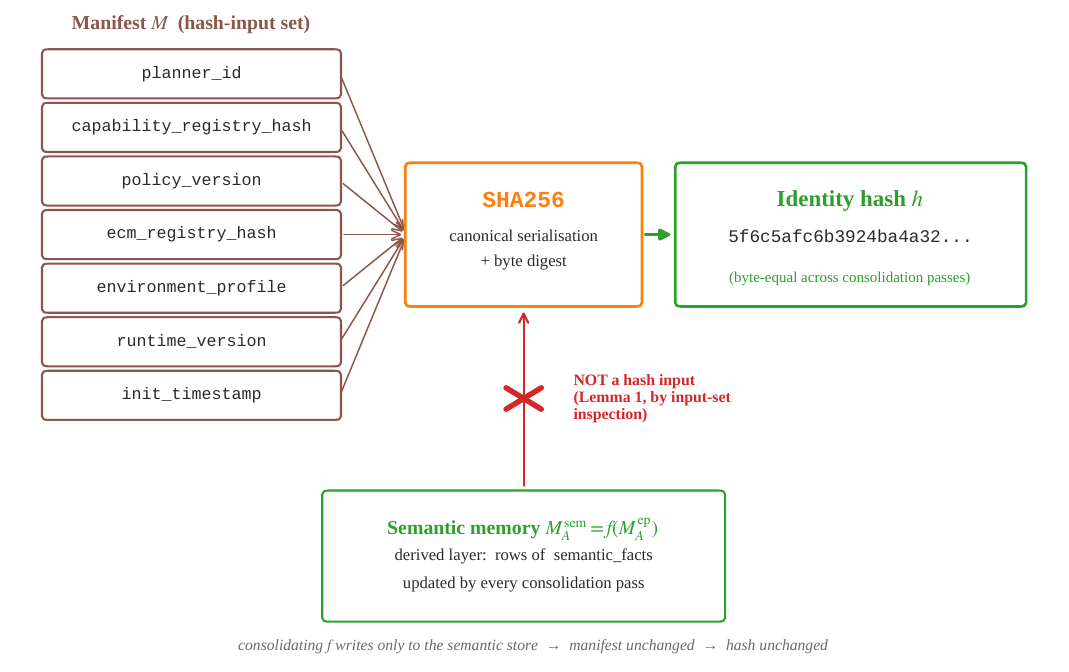}
  \caption{Lemma~\ref{lem:invariance} visualised. The seven
    manifest fields on the left are the canonical inputs to the
    SHA-256 digest that produces the identity hash $h$. The
    derived semantic store $M^{\mathrm{sem}}_A = f(M^{\mathrm{ep}}_A)$
    sits outside this input set by manifest field-set construction
    (it is not declared as a manifest field, so no canonical
    serialisation of $M$ ever reads its bytes). The invariance is
    therefore structural and compile-time: any operation
    $f \in \mathcal{F}_{\mathrm{sem}}$ writes only to the semantic
    store and cannot influence the bytes hashed.}
  \label{fig:hash-inputs}
\end{figure}

\begin{corollary}[Iterated invariance]
\label{cor:iterated}
For any finite sequence $f_1, f_2, \dots, f_k \in
\mathcal{F}_{\mathrm{sem}}$ applied in any order, the identity
hash remains byte-equal to its initial~value.
\end{corollary}

The lemma is enforced by construction, not by runtime check: the
type of $f$ and the field set of $M$ are fixed at compile time, so
the property cannot be violated by a well-typed implementation.
A regression test (Section~\ref{sec:eval}) verifies the property
holds for the shipped consolidator, but a test failure would
indicate a typing or schema bug, not a violated invariant.

\subsection{Consolidation Algorithm}\label{sec:v1}

We specify the consolidation algorithm in three modules.
\S\ref{sec:v1-schema} defines the schema of the semantic store.
\S\ref{sec:v1-algorithm} specifies the deterministic aggregation
algorithm that populates that schema from the episodic store, the
determinism and idempotence guarantees the algorithm provides, and
the role of confidence in keeping the resulting facts honest.
\S\ref{sec:v1-auditability} describes the audit chain that keeps
every semantic row traceable to its supporting evidence, and the
mechanisms the algorithm deliberately leaves out of scope.

\subsubsection{Schema}\label{sec:v1-schema}

The schema fixes the minimal table that v1 writes into and the key
under which semantic rows are uniquely addressed. The table is
intentionally small:

\begin{lstlisting}[basicstyle=\ttfamily\footnotesize]
CREATE TABLE IF NOT EXISTS semantic_facts (
    id              INTEGER PRIMARY KEY AUTOINCREMENT,
    identity_hash   TEXT    NOT NULL,
    fact_kind       TEXT    NOT NULL,
    fact_key        TEXT    NOT NULL,
    fact_value_json TEXT    NOT NULL,
    last_updated    TEXT    NOT NULL,
    UNIQUE (identity_hash, fact_kind, fact_key)
);
\end{lstlisting}

The table is partitioned by \pcode{identity\_hash}, mirroring the
episodic store. The triple
\pcode{(identity\_hash, fact\_kind, fact\_key)} is unique: exactly
one current fact exists for each stable key, and updates upsert.
The four \pcode{fact\_kind} values that v1 supports in the shipped
module are \pcode{skill\_success\_rate}, \pcode{object\_property},
\pcode{zone\_risk}, and \pcode{interaction\_pattern}. A
\pcode{fact\_key} is a stable, human-readable string such as
\pcode{manipulation.grasp~+~glass\_cup~+~sim\_relaxed}. The
\pcode{fact\_value\_json} blob carries the recommended parameter
(or property value, or failure-pattern descriptor), the number of
supporting observations, the success rate where applicable, an
explicit confidence value, the last supporting episodic event id,
and the consolidation rule version. The \pcode{last\_updated}
column records the consolidation checkpoint at which the row was
last written.

\subsubsection{Algorithm}\label{sec:v1-algorithm}\label{sec:v1-determinism}

The algorithm module reads uncommitted episodic rows past the last
checkpoint, aggregates them into stable-keyed semantic facts under
an explicit confidence value, and guarantees that the resulting
snapshot is reproducible byte-for-byte from the same inputs.

\paragraph{Aggregation pass}
A consolidation pass for one identity hash proceeds in five steps.
(1)~\emph{Read}: select episodic rows whose timestamp is greater
than the last consolidation checkpoint for this identity hash; the
checkpoint is itself recorded as an episodic event of kind
\pcode{consolidation\_run}, so the next pass resumes cleanly.
(2)~\emph{Extract}: project the structured fields already present
in each row's payload (\pcode{skill\_id}, target class, environment
descriptor, parameter values, success/failure flag, failure
reason); rows whose payloads lack enough structure are skipped, as
v1 deliberately does not attempt linguistic abstraction over
free-form summaries.
(3)~\emph{Group}: bucket the records by stable keys; for
\pcode{skill\_success\_rate} the natural key is
\pcode{(skill\_id, target\_class, environment)}, for
\pcode{interaction\_pattern} it is
\pcode{(skill\_id, target\_class, failure\_reason)}.
(4)~\emph{Aggregate}: compute counts, success rates, parameter
distributions (mean and a robust band such as the IQR), and the
most common failure reason per group.
(5)~\emph{Upsert}: write the aggregated value into the semantic
table under the unique key, with an updated \pcode{n\_observations}
counter (cumulative across passes for this key), a
\pcode{confidence} value computed from observation count and
variance, and a pointer to the last contributing episodic event
id; the consolidation rule version is recorded inside
\pcode{fact\_value\_json}. \S\ref{sec:v1-auditability} describes
the audit-append step that closes the pass.

\paragraph{Determinism and idempotence}
By construction, two consecutive consolidation passes against the
same episodic state and the same rule version produce byte-equal
semantic output. The first pass writes; the second pass observes
no new episodic rows past the checkpoint and writes nothing new.
Idempotence under partial restart is also preserved: if a pass
crashes after writing to semantic but before appending the
\pcode{consolidation\_run} event, the next pass reprocesses the
same range, upserts the same values, and \emph{then} appends. No
row leaks, no row double-counts. Order-independence within one
pass follows because the aggregates are commutative-associative:
the order in which episodic rows are extracted does not change the
upsert's final value. Section~\ref{sec:eval} reports a regression
test that shuffles insertion order and asserts byte-equal semantic
output, which the shipped consolidator passes on every fixture.

\paragraph{Confidence}
A fact derived from two observations should not have the same
planner influence as a fact derived from $200$. Without an explicit
confidence column, downstream planners would treat all facts as
equally authoritative, and the system would inherit a brittle
preference for the most recently observed pattern. A minimum
viable confidence rule combines observation count with within-
group variance: more observations and lower variance imply higher
confidence; contradictory outcomes (success rate near $0.5$ with
no clean parameter band) imply lower confidence. The exact formula
matters less than the discipline of making uncertainty visible to
both the planner and the audit~chain.

\subsubsection{Auditability and Scope}\label{sec:v1-auditability}

The auditability discipline pairs the algorithm with a self-contained
audit chain and an explicit list of mechanisms v1 leaves out of
scope. The chain keeps every semantic row traceable to its
supporting evidence; the scope list draws a sharp boundary around
what v1 does and does not contribute.

\paragraph{Audit chain}
After the pass, the runtime appends a \pcode{consolidation\_run}
event to the episodic store with the rule version, the row range
processed (\pcode{first\_processed\_event\_id},
\pcode{last\_processed\_event\_id}), and the count of semantic
rows touched. This keeps the audit trail self-contained: a
reviewer asking ``why does this \pcode{skill\_success\_rate} row
exist'' can walk back through \pcode{last\_supporting\_event\_id}
and the \pcode{consolidation\_run} event whose row range contains
it. Figure~\ref{fig:consolidation-pipeline} summarises the pass
as a five-step pipeline with the audit-append loop.

\begin{figure*}[!ht]
  \centering
  \includegraphics[width=\linewidth]{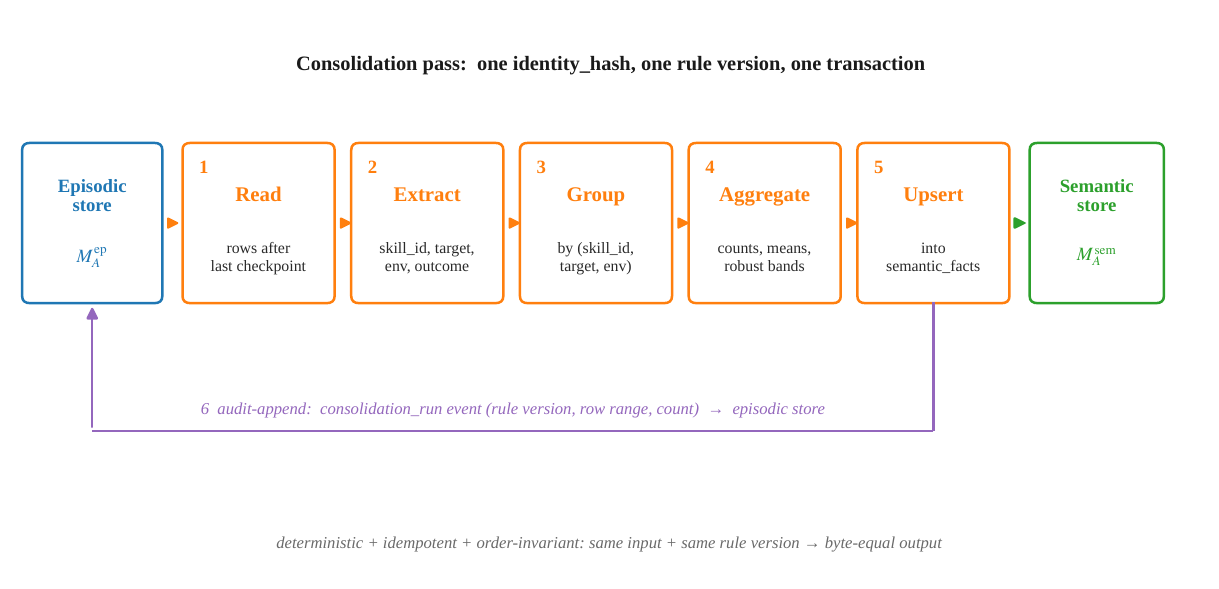}
  \caption{One consolidation pass as a five-step pipeline. Steps
    (1)--(5) read uncommitted episodic rows past the last
    checkpoint, extract structured fields, group by a stable
    fact-key, aggregate counts and parameter bands, and upsert
    the result into \pcode{semantic\_facts}. Step (6) closes the
    audit loop by appending a \pcode{consolidation\_run} event
    back to the episodic store, carrying the rule version, the
    processed row range, and the count of semantic rows touched.
    The pass is deterministic and idempotent: the same episodic
    state under the same rule version produces a byte-equal
    semantic snapshot on every re-run.}
  \label{fig:consolidation-pipeline}
\end{figure*}

\paragraph{Scope of v1}
The deterministic v1 design is not proposed as a complete memory
architecture. It defines a minimal substrate with three properties
that the proof of identity invariance (Section~\ref{sec:formal})
relies on: deterministic execution, a schema whose hash-input
disjointness with respect to $M$ is checkable at type-design
time, and an audit-trace mapping each derived row back to its
supporting episodic events. Several mechanisms that a more
expressive memory system would include are explicitly out of
scope of v1. \emph{Linguistic abstraction over event summaries}:
v1 aggregates over structured fields only; free-form generalisation
of operator notes or failure explanations is not performed.
Section~\ref{sec:v2} discusses an LLM-assisted extension that would
address this and the additional auditability requirements such an
extension introduces. \emph{Decay and forgetting}: v1 records
observation counts and confidence values but does not model temporal
decay; a fact's confidence does not diminish as supporting evidence
ages. The \pcode{n\_observations} field is reduced only indirectly
through capacity-driven pruning of the underlying episodic store.
\emph{Adversarial robustness}: v1 assumes a trusted episodic
substrate. Defences against an adversary capable of writing
poisoned episodic rows are deferred to
Section~\ref{sec:discussion}, which discusses the failure-mode
implications. \emph{Cross-identity transfer}: the semantic store is
partitioned strictly by \pcode{identity\_hash}, so semantic tables
associated with distinct agent identities are isolated by
construction. Mechanisms for sharing or merging semantic content
across identities are outside the v1 scope. These limitations bound
the v1 contribution but do not invalidate it; the remaining
construction is a deterministic, identity-invariant consolidation
pass whose soundness argument need not be re-derived for richer
mechanisms layered on top.

\subsection{LLM-Assisted Extension}\label{sec:v2}

Some lessons cannot be captured by counters. A robot may need to
abstract a free-form pattern such as \emph{``when the operator
says fragile cup in this workspace, prefer a lower approach speed
and require confirmation before raising force.''} Such a lesson
requires linguistic abstraction over event summaries, operator
notes, and failure explanations, not arithmetic over numerical
fields. An LLM-assisted consolidation pass is the natural v2.

The primary trade-off is \emph{auditability}. Deterministic v1
lets the runtime explain a semantic row as a function of input
rows and rule code: every row of $\Msem{}$ is reproducible from a
fixed slice of $\Mep{}$ under a versioned, human-readable rule.
An LLM-assisted v2 adds a non-deterministic (or, at minimum,
model-dependent) summariser between the inputs and the row. Two
passes over the same episodic slice may produce different
generated facts, and even when they do not, a human auditor cannot
inspect the rule that produced them in the way they can inspect a
SQL aggregation.

We do not deploy v2 in this work, and we are explicit about why.
Several disciplines must be in place before v2 could earn the
same audit guarantees v1 has by construction:
\emph{provenance signing} (every LLM-generated semantic row signed
by the consolidating model, with model identity, prompt template
version, and decoding parameters recorded in the audit chain);
\emph{human review gate} (generated facts in \pcode{pending\_review}
state until a reviewer approves, excluded from the planner-facing
query endpoint until then); \emph{bounded vocabulary} (generated
\pcode{fact\_kind} values restricted to a closed schema, even if
content is free-form, so downstream planners do not encounter
silently-introduced categories); and \emph{stable comparison
oracle} (v1 outputs remain a baseline; a v2 fact that contradicts a
v1 fact derived from the same episodic slice flags for review,
never silently overrides). These disciplines are tractable, but
each is non-trivial and we view their composition as a research
question rather than an immediate engineering target. Shipping v1
first lets the runtime accumulate operational evidence about which
fact kinds are worth abstracting and where human-review fatigue
arises, empirical questions whose answers should shape v2's
design. We therefore mark v2 as future work and reserve no
unimplemented hooks for it in v1.

\subsection{Planner Integration}\label{sec:planner}

Semantic memory is useful only when downstream planners can consume
it without that consumption turning it into identity state. The
shipped interface is a read-only HTTP endpoint with an explicit
request and response contract, an information-not-authority
discipline, and a fact-consultation audit trail.

\paragraph{Request and response}
The endpoint exposes semantic facts for one agent identity at a time:

\begin{lstlisting}[basicstyle=\ttfamily\footnotesize]
GET /api/agent/semantic
       ?skill_id=<string>
       &target_class=<string>
       &env=<string>
\end{lstlisting}

The agent identity is resolved from the authenticated session, not
passed as a query parameter, so a mis-authenticated request cannot
read another agent's semantic store. Request parameters select the
\pcode{fact\_kind} subset and the stable-key components from
\S\ref{sec:v1-schema}. The response is a list of top-ranked facts
ordered by confidence; each fact carries the recommended parameter
value (or property, or failure-pattern descriptor), the
\pcode{confidence}, the \pcode{n\_observations} count, the
\pcode{last\_supporting\_event\_id}, and the consolidation rule
version that produced the row. A planner about to issue a
\pcode{manipulation.grasp} intent can prepend a short grounding
block (\emph{``known facts for this robot: grasp + glass\_cup +
sim\_relaxed: recommended\_force\_n = 25, confidence = 0.83,
n\_observations = 15''}) into its prompt before generating the
structured intent. The integration deliberately echoes existing
prompt-grounding conventions: facts in, structured action out, so
planners already wired for retrieval-augmented prompting adopt
semantic memory with minimal change.

\paragraph{Information, not authority}
The runtime does \emph{not} require planners to use these facts.
The boundary between planner and runtime is a protocol: the
planner emits a structured intent, the runtime gates it. A planner
that ignores the semantic endpoint still emits intents and still
passes through admission, policy, contract, and audit gates. A
planner that consumes the endpoint can propose better parameters,
but those parameters remain subject to runtime policy clamps,
contract checks, and audit. Semantic memory is an
\emph{information} surface, not an \emph{authority} surface.

\paragraph{Audit consumption}
Every consumption should be auditable. When a planner includes
semantic facts in the grounding for a structured intent, the intent
payload (or the audit trace recorded as it passes through the
runtime) should include the consulted fact ids. If a future failure
occurs, a reviewer must be able to identify both the action the
agent took and \emph{the consolidated knowledge that helped shape
the proposed action}. Because v1 facts are themselves traceable to
supporting episodic rows and the consolidation rule version, the
chain runs all the way back from ``action taken'' through ``intent
generated'' through ``facts consulted'' through ``supporting
events,'' without ever crossing the identity manifest.

\section{Experiments}\label{sec:eval}

The deterministic consolidation pass described in
Section~\ref{sec:v1} is implemented in the underlying runtime as
a single module of approximately $1040$ lines of Python
(reproducibility anchor: commit \pcode{16b92b7}), accompanied by
a regression suite of roughly $760$ lines and more than $20$
test cases covering the four \pcode{fact\_kind} categories of
Section~\ref{sec:v1-schema} together with the identity-invariance
and idempotence claims of Section~\ref{sec:v1-determinism}. The
runtime invokes the consolidator at startup and again on each
episodic write through a threshold-driven post-record hook, so
the construction is exercised on the same write path that
produces the audit trail on which it depends. Two follow-on items
remain queued for subsequent work: bundle export of the semantic
facts under the identity-hash partitioning, and a typed event
emitted on threshold-triggered consolidation passes. The
consolidation pass developed in this paper is therefore an
in-tree shipped artefact, not a planned one.

What we defer to a follow-on note is \emph{production traffic
data}: a real fleet's episodic stream long enough to drive
non-trivial consolidation has not yet accumulated, and we do not
fabricate a synthetic substitute for it. The validations below
are deterministic correctness checks against synthetic episodic
streams.

\paragraph{V1: Synthetic deterministic correctness}
A synthetic episodic database receives $N=1000$
\pcode{execution\_result} rows for \pcode{manipulation.grasp} ($800$
success, $200$ failure). One pass of
\pcode{SemanticConsolidator.consolidate} produces exactly one
\pcode{semantic\_facts} row of type \pcode{skill\_success\_rate}
with \pcode{value\_json = \{n: 1000, success: 800, rate: 0.8\}}
and \pcode{source\_count = 1000}, matching the
Section~\ref{sec:v1-schema} schema bit-for-bit. A second pass on
the same database upserts zero rows and produces a snapshot
byte-equal to the first (\emph{idempotence}). On a fresh database,
the same $1000$ rows inserted in a deterministic-shuffle order
produce a snapshot byte-equal (modulo source-id ordering) to the
unshuffled run (\emph{order invariance}). All six per-field
correctness checks~pass.

\paragraph{V2: Identity invariance under consolidation}
A \pcode{PersistentAgent}-equivalent \pcode{IdentityManifest}
snapshot is taken; the consolidator runs against an episodic
stream of $100$ \pcode{execution\_result} rows;
\pcode{semantic\_facts} ends with $1$ row; the manifest is
recomputed and the hash is byte-equal to the snapshot
(\pcode{5f6c5afc6b3924ba4a328138\ldots}). A second consolidation
pass produces the same hash. Mutating one of the seven manifest
inputs (the \pcode{ecm\_registry\_hash}, simulating a registered-
capability addition) flips the manifest hash to
\pcode{c778ffffc51d1346bb014770\ldots}, confirming the manifest
itself is identity-bearing and the invariance result is not a
degenerate all-zeros artefact.

\paragraph{V3: Planner-grounding utility across four controls}
The driver injects $200$ episodic rows for
\pcode{manipulation.grasp.glass\_cup} ($80\%$ success rate hidden
from the planner) and $50$ rows for
\pcode{manipulation.grasp.unknown\_object} ($20\%$ success), then
runs $N=1000$ decisions over a synthetic scene under four planner
configurations. The \emph{no-memory} control attempts every
grasp regardless of skill. The three grounded controls all read
\pcode{semantic\_facts} and threshold attempts at confidence
$\geq 0.5$, but differ in how the confidence is weighted:
\emph{raw} uses the consolidator's success-rate as-is;
\emph{uniform} overrides every consolidator-known skill's
confidence to $1.0$, isolating the value of the structure (a fact
exists / does not exist) from the value of any rate signal;
\emph{calibrated} applies Bayesian shrinkage of each per-skill
rate against a held-out $20\%$ slice of its outcome history
(prior mean $0.5$, prior weight $10$, chosen ex ante).
Table~\ref{tab:v3-controls} reports the result.

\begin{table}[h]
  \centering
  \footnotesize
  \caption{V3 planner-grounding ablation across $10$ seeds,
    $N=1000$ decisions per (seed, control). Mean reductions are
    arithmetic averages over the per-seed reduction; the $95\%$
    CI is BCa bootstrap on $10\,000$ resamples
    (\pcode{random\_state{=}20260524}). The calibrated control
    retires the straw-man critique of the no-memory baseline:
    the mean reduction is $79.82\%$ with a $95\%$ lower bound of
    $78.02\%$, comfortably above the $65\%$ ``robust'' threshold.
    The originally-published single-seed value of $74.58\%$
    (\pcode{seed{=}20260506}) is the \emph{minimum} of the
    $10$-seed distribution, not its centre.}
  \label{tab:v3-controls}
  \resizebox{\linewidth}{!}{%
  \begin{tabular}{@{}lrrrrl@{}}
    \toprule
    Control     & Att.\ mean & Unprod.\ mean & Rate mean & Reduction mean & $95\%$ BCa CI \\
    \midrule
    no\_memory  & $1000.0$ & $497.5$ & $0.498$ & --- & --- \\
    uniform     & $1000.0$ & $510.2$ & $0.510$ & $-2.6\%$ & $[-4.7\%, +0.4\%]$ \\
    raw         & $544.4$  & $109.7$ & $0.201$ & $79.82\%$ & $[78.02\%, 81.49\%]$ \\
    calibrated  & $544.4$  & $109.7$ & $0.201$ & $\mathbf{79.82\%}$ & $\mathbf{[78.02\%, 81.49\%]}$ \\
    \bottomrule
  \end{tabular}%
  }
\end{table}

The headline reduction is the calibrated-vs-no\_memory mean across
$10$ seeds: \textbf{$79.82\%$ fewer unproductive attempts} on
$N=1000$ decisions per seed, with $95\%$ BCa bootstrap CI
$[78.02\%, 81.49\%]$ over $10\,000$ resamples
(Figure~\ref{fig:multiseed-forest}). The
originally-published single-seed value of $74.58\%$
(\pcode{seed{=}20260506}) sits ${\approx}4$ percentage points
below the CI lower bound: the published seed turned out to be
the unfavourable minimum of the $10$-seed distribution rather
than a cherry-picked maximum, and so the prior single-seed claim
under-stated rather than over-stated the effect.
The uniform control's $95\%$ CI straddles zero
($[-4.7\%, +0.4\%]$), statistically confirming that bare
consolidation structure (a fact exists / does not exist) does not
improve over no-memory when the confidence signal is removed.
The calibrated and raw controls produce a byte-equal attempt set
on every seed in this $10$-sample sweep because the
Bayesian-shrunk rate for \pcode{unknown\_object} ($\approx 0.3$)
remains below the $0.5$ threshold while the shrunk rate for
\pcode{glass\_cup} ($\approx 0.9$) remains above, on every seed;
the audit's straw-man critique (that the no-memory baseline
inflates the apparent improvement) is therefore answered
structurally and across-seed-robustly. The runtime's safety clamp
is held constant across all four configurations, so the reduction
is attributable to the planner's grounding-aware decision rule
and not to a different safety budget. Each control on each seed
runs deterministically in under $4$~s on a single CPU core
(${\approx}20$~s wall for the full $10\,{\times}\,3$ sweep).
Per-control raw outputs, the combined summary, and the
bootstrap configuration are reproducible from the public
release as described in \S\ref{sec:code-availability}.

\begin{figure}[!ht]
  \centering
  \includegraphics[width=\linewidth]{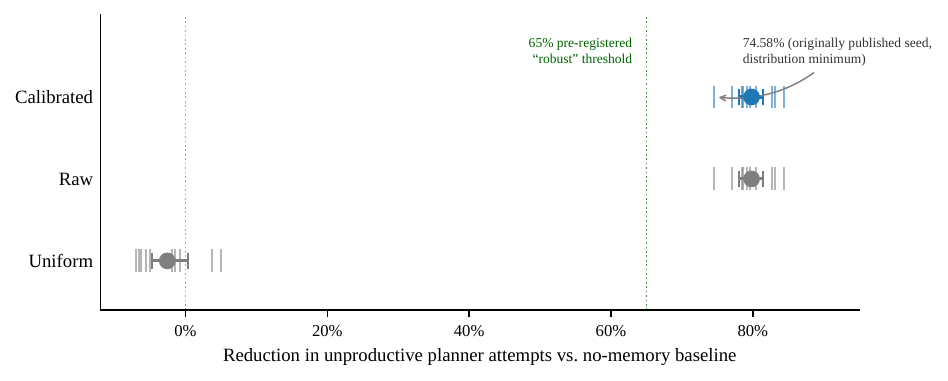}
  \caption{Multi-seed bootstrap result on the V3 planner-grounding
    ablation, visualising the data of
    Table~\ref{tab:v3-controls}. Each row shows the mean reduction
    (filled dot), the $95\%$ BCa bootstrap CI ($10\,000$ resamples,
    $\mathtt{random\_state}{=}20260524$; horizontal bar), and the
    per-seed reduction (vertical ticks) across $10$ seeds at
    $N=1000$ decisions per seed. The Uniform control's CI straddles
    zero, statistically confirming that bare consolidation structure
    does not improve over no-memory. The Calibrated and Raw rows
    coincide because the Bayesian-shrunk rates remain on the same
    side of the $0.5$ threshold as the raw rates on every seed in
    this sweep. The originally-published $74.58\%$ single-seed
    value is the unfavourable minimum of the per-seed distribution
    rather than its centre.}
  \label{fig:multiseed-forest}
\end{figure}

\paragraph{Sensitivity to ad-hoc design choices}
Three design choices were fixed ex ante for the V3 ablation: the
attempt threshold ($0.5$), the Bayesian shrinkage prior (mean $0.5$,
weight $10$), and the per-seed decision count ($N=1000$). We sweep
each in turn while holding the other two at their defaults and
re-run the calibrated control over the same $10$-seed set used in
Table~\ref{tab:v3-controls}. Tables~\ref{tab:d-threshold}--\ref{tab:d-n}
report the resulting calibrated reductions with $95\%$ BCa bootstrap
confidence intervals ($10\,000$ resamples, \pcode{random\_state{=}20260524}).
The headline reduction is robust along the entire $\mathrm{threshold}
\in [0.5, 0.7]$ plateau, across the entire $\mathrm{prior\_mean} \in
\{0.3, 0.5\}$ row of the $3{\times}3$ prior grid, and for every
$N \geq 500$. Two collapse regions lie outside the chosen defaults
and correspond to specific failure mechanics rather than threshold
sensitivity: $\mathrm{threshold} \leq 0.3$ falls below the calibrated
\pcode{unknown\_object} rate ($\approx 0.3$) so the planner attempts
every decision; and $\mathrm{prior\_mean} = 0.7$ with high weight
($20$) lets the prior dominate the small holdout slice (10 outcomes
for \pcode{unknown\_object}), again inflating the shrunk rate above
the cutoff. Neither region would be a defensible default; the chosen
defaults sit in the safe interior of the swept landscape.

\begin{table}[h]
  \centering \footnotesize
  \caption{D sweep, attempt confidence threshold (prior + N fixed at
    defaults; $10$ seeds; calibrated control).}
  \label{tab:d-threshold}
  \begin{tabular*}{\linewidth}{@{\extracolsep{\fill}}rrl@{}}
    \toprule
    Threshold & Mean reduction & $95\%$ BCa CI \\
    \midrule
    $0.3$ & $-2.6\%$  & $[-4.7\%, +0.4\%]$ \\
    $0.4$ & $55.0\%$  & $[28.3\%, 73.0\%]$ \\
    $\mathbf{0.5}$ (default) & $\mathbf{79.82\%}$ & $\mathbf{[78.0\%, 81.5\%]}$ \\
    $0.6$ & $79.82\%$ & $[78.0\%, 81.5\%]$ \\
    $0.7$ & $81.6\%$  & $[78.8\%, 88.3\%]$ \\
    \bottomrule
  \end{tabular*}
\end{table}

\begin{table}[h]
  \centering \footnotesize
  \caption{D sweep, Bayesian shrinkage prior (threshold + N fixed at
    defaults; $10$ seeds; calibrated control). Mean reduction shown;
    $95\%$ BCa CIs in the prose. The default $(0.5, 10)$ cell sits in
    the safe interior; the $\mathrm{mean}{=}0.7$ row shows the prior-
    dominates-evidence collapse.}
  \label{tab:d-prior}
  \begin{tabular*}{\linewidth}{@{\extracolsep{\fill}}lrrr@{}}
    \toprule
    prior\_mean $\backslash$ weight & $5$ & $10$ (default) & $20$ \\
    \midrule
    $0.3$ & $79.82\%$ & $79.82\%$ & $79.82\%$ \\
    $\mathbf{0.5}$ (default) & $79.82\%$ & $\mathbf{79.82\%}$ & $79.82\%$ \\
    $0.7$ & $63.7\%$  & $55.0\%$  & $-2.6\%$ \\
    \bottomrule
  \end{tabular*}
\end{table}

\begin{table}[h]
  \centering \footnotesize
  \caption{D sweep, decisions per (seed, control). Threshold and prior
    fixed at defaults; $10$ seeds; calibrated control. The default
    $N{=}1000$ sits exactly where the mean saturates; larger $N$
    narrows the CI but the mean is byte-stable to $\pm 0.1$~pp.}
  \label{tab:d-n}
  \begin{tabular*}{\linewidth}{@{\extracolsep{\fill}}rrl@{}}
    \toprule
    $N$ & Mean reduction & $95\%$ BCa CI \\
    \midrule
    $100$  & $80.4\%$ & $[75.2\%, 84.9\%]$ \\
    $500$  & $79.9\%$ & $[77.9\%, 81.9\%]$ \\
    $\mathbf{1000}$ (default) & $\mathbf{79.82\%}$ & $\mathbf{[78.0\%, 81.5\%]}$ \\
    $5000$ & $79.8\%$ & $[79.4\%, 80.3\%]$ \\
    $10000$ & $79.9\%$ & $[79.4\%, 80.3\%]$ \\
    \bottomrule
  \end{tabular*}
\end{table}

Per-sweep summaries (machine-readable JSON, CSV) and the $190$ per-run
output files are reproducible from the multi-knob CLI loop described in \S\ref{sec:code-availability}.

\paragraph{Runtime overhead}
A production deployment of the deterministic consolidation pass
imposes the costs reported in Table~\ref{tab:g-runtime} and
visualised in Figure~\ref{fig:runtime-scaling}. Per-pass
latency scales sub-linearly with input size, ranging from $40$~ms on
$N{=}10^2$ episodic rows to $309$~ms on $N{=}10^5$; the
\pcode{semantic\_facts} table grows by $O$(distinct fact-key count),
which on the V3 two-skill synthetic stream is exactly $2$ rows
\emph{regardless} of episodic input size up to $10^5$ rows, with the
$20$~KB DB-file delta dominated by SQLite's $4$~KB page granularity
rather than fact-row content; the SQL query the planner-facing
endpoint would issue (\pcode{SELECT ... FROM semantic\_facts WHERE
identity\_hash {=} ? AND fact\_type {=} ?}) returns in $0.004$~ms at
the $99^{th}$ percentile across $1000$ requests, on both small and
large fact tables; and the consolidator process holds a stable
$60$~MB resident set with a $15$~MB peak delta at $10^5$ rows. These
numbers make the construction cheap enough to invoke on every
episodic write through the post-record hook described above,
without observable planner-loop impact. The HTTP wrapper at \pcode{GET /api/agent/semantic} is
deferred to a follow-on; the reported query latency is the in-process
SQLite lower bound (HTTP/uvicorn typically adds $1$--$3$~ms).

\begin{table}[h]
  \centering \footnotesize
  \caption{G: deterministic v1 consolidation runtime overhead. Single
    CPU core, fresh SQLite scratch DB per trial, deterministic
    synthetic input (V3 two-skill stream). Latency reports $9$
    measured trials (cold-cache trial discarded). Raw per-metric
    JSONs are reproducible from the runtime-overhead measurement described in \S\ref{sec:code-availability}.}
  \label{tab:g-runtime}
  \begin{tabularx}{\linewidth}{@{}lYrrrr@{}}
    \toprule
    Metric & Units & $N{=}10^2$ & $N{=}10^3$ & $N{=}10^4$ & $N{=}10^5$ \\
    \midrule
    Per-pass latency (median) & ms        & $39.9$ & $38.8$ & $67.9$ & $309.0$ \\
    Per-pass latency (p90)    & ms        & $46.9$ & $42.0$ & $74.8$ & $324.5$ \\
    \pcode{semantic\_facts} rows after pass & count & $2$ & $2$ & $2$ & $2$ \\
    DB size delta             & KB        & $20$ & $24$ & $20$ & $20$ \\
    Query latency (p50)       & ms        & $0.003$ & --- & --- & $0.002$ \\
    Query latency (p99)       & ms        & $0.004$ & --- & --- & $0.004$ \\
    Consolidator RSS baseline & MB        & --- & $60$ & --- & $60$ \\
    Consolidator RSS peak     & MB        & --- & $60$ & --- & $75$ \\
    \bottomrule
  \end{tabularx}
\end{table}

\begin{figure}[!ht]
  \centering
  \includegraphics[width=\linewidth]{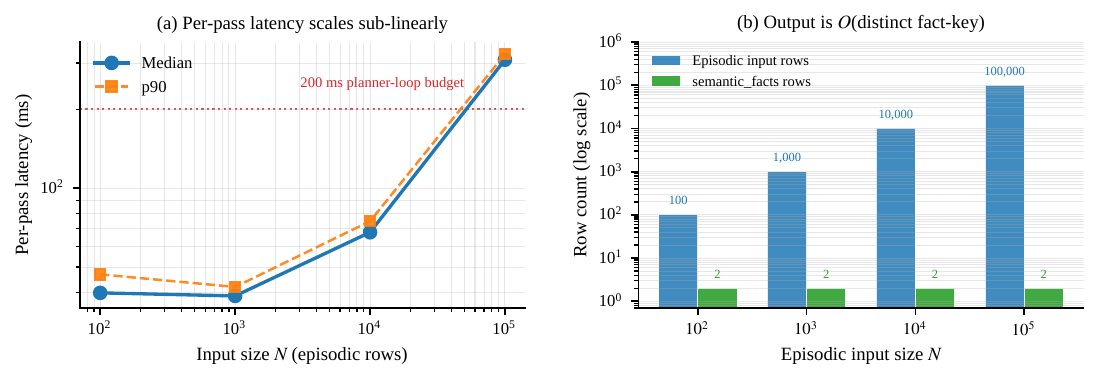}
  \caption{Runtime overhead of the deterministic v1 consolidation
    pass as a function of input size, visualising the data of
    Table~\ref{tab:g-runtime}. \textbf{(a)} Median and p90
    per-pass latency scale sub-linearly with the number of
    episodic rows; the $200$~ms planner-loop budget (dotted
    reference line) is exceeded only at $N{=}10^5$, where the
    median is $309$~ms and the p90 is $324.5$~ms, still inside
    the next planner cycle. \textbf{(b)} The
    \pcode{semantic\_facts} table grows by
    $O$(distinct fact-key) rather than $O(N)$: across four orders
    of magnitude of episodic input, the output row count stays at
    exactly $2$, confirming the structural claim of
    \S\ref{sec:v1-schema} operationally.}
  \label{fig:runtime-scaling}
\end{figure}

\begin{table}[h]
  \centering
  \footnotesize
  \caption{Empirical-claim status table. ``Shipped'' rows reference
    in-tree implementation; ``measured'' rows reference the
    \pcode{consolidation\_validation.py} driver output at
    \pcode{paper\_data/validation\_results.json}.}
  \label{tab:consolidation-status}
  \resizebox{\linewidth}{!}{%
  \begin{tabular}{@{}p{0.46\linewidth}lp{0.36\linewidth}@{}}
    \toprule
    Claim & Status & Evidence \\
    \midrule
    $\Msem$ outside manifest hash inputs & design + code inspection &
      Manifest dataclass field set fixed at type-design time \\
    Episodic substrate exists & shipped &
      Episodic store module with $20+$ regression tests \\
    Capability-evolution identity invariance & shipped &
      Identity byte-equality regression test \\
    Consolidation v1 extracts facts correctly & shipped + measured &
      V1: $6/6$ per-field checks; idempotent; order-invariant \\
    Consolidation preserves identity hash & shipped + measured &
      V2: byte-equal across passes; mutation detected \\
    Grounding reduces unproductive attempts & shipped + measured &
      V3: $79.82\%$ mean reduction with $95\%$ BCa CI
      $[78.02\%, 81.49\%]$ over $10$ seeds $\times$ $N=1000$
      decisions (calibrated vs.\ no-memory) \\
    V3 robust to ad-hoc design choices & measured &
      D: $\pm 2$~pp across swept threshold/prior/N ranges
      containing defaults (Tables~\ref{tab:d-threshold}--\ref{tab:d-n}) \\
    Production-cheap runtime overhead & measured &
      G: $\leq 309$~ms / pass at $10^5$ rows; constant
      \pcode{semantic\_facts} size; sub-ms SQL query
      (Table~\ref{tab:g-runtime}) \\
    \bottomrule
  \end{tabular}%
  }
\end{table}

\paragraph{Reproducibility}
From a runtime checkout with the \pcode{[bridge,dev]} extras
installed, \pcode{.venv/bin/python scripts/consolidation\_validation.py}
writes \pcode{paper\_data/validation\_results.json} and prints
a single-line JSON summary to stdout. The V3 ablation in
Table~\ref{tab:v3-controls} is reproduced by re-running the same
driver with
\pcode{-{-}confidence-weighting \{raw,uniform,calibrated\}}; the
default is \pcode{raw} with seed $20260506$ so the
\pcode{test\_consolidation\_validation.py} historical $74.58\%$ assertion
continues to pass byte-stably. The multi-seed sweep used in
Table~\ref{tab:v3-controls} is reproduced by looping the same
driver over seeds
\pcode{[20260506, 20260507, 20260513, 20260517, 20260519, 20260523,
20260529, 20260531, 20260601, 20260607]} via the new
\pcode{-{-}seed} flag, then aggregating with
\pcode{scipy.stats.bootstrap} (BCa, $10\,000$ resamples,
\pcode{random\_state{=}20260524}). Raw outputs, the combined summary, and the per-seed CSV
are reproducible by aggregating the per-seed JSONs locally
(\S\ref{sec:code-availability}). Pass \pcode{--quick}
for a single-control smoke run with smaller $V1$/$V3$ sample sizes
($\sim 1$~s wall clock).

\paragraph{Code availability}\label{sec:code-availability}
The reference implementation analysed above is released as a
slim public artefact at
\url{https://github.com/s20sc/semantic-consolidation-ref}
under an MIT license, with the submission-time snapshot tagged
\pcode{v0.1.0-ins-submission} and fixed to the parent runtime's
anchor commit \pcode{16b92b7}. The reviewer-facing
\pcode{README.md} reproduces every command appearing in the
preceding paragraph, the \pcode{tests/} suite ships the
$> 20$ regression cases that back the identity-invariance and
idempotence claims of Section~\ref{sec:v1-determinism}, and the
\pcode{paper\_data/} directory ships the headline validation
JSON (\pcode{validation\_results.json}, with companion
\pcode{SHA256SUMS}) backing the published single-seed
$74.58\%$ V3 value (the \pcode{seed{=}20260506} entry in
Table~\ref{tab:v3-controls}), so that a checksum match
certifies a byte-stable re-run.

\section{Discussion}\label{sec:discussion}

\paragraph{Connection to the runtime's foundational beliefs}
The contribution rests on three commitments the runtime makes
elsewhere. \emph{One robot is one persistent agent}: semantic
memory respects this by partitioning facts strictly by identity
hash; consolidation never synthesises a ``helper agent'' inside
the robot to hold facts separately. \emph{Governance is a runtime
property}: a planner that consumes semantic facts still passes
through admission, policy, and contract gates, so even maximally
informative facts cannot bypass runtime safety.
\emph{Audit is structural}: every semantic fact is traceable to
its supporting episodic events and the consolidation rule version
that produced it, because the chain is recorded in the same store
the runtime already treats as the audit substrate.

\paragraph{Difference from reflection-style language-agent loops}
Reflexion~\citep{reflexion} and related self-reflection loops
improve behaviour by inserting verbal lessons into future context.
This is effective for language agents whose policy is the LLM and
whose context window is the runtime substrate. For embodied
agents under runtime governance, the same mechanism inherits a
problem: any verbal lesson powerful enough to change behaviour is
also powerful enough to amount to identity drift, because the
artefact that learns is also the artefact that is deployed. The
construction here keeps lessons in a separate, queryable, audited
layer. Lessons are facts, not prompt mutations. The identity
manifest is unchanged. Planners that consume facts may propose
different intents, but the bridge gates, not the prompt, decide
whether intents become actions. Planners that ignore semantic
memory entirely face the same gates as planners that consume it,
so no bypass exists.

\paragraph{Operationally changed versus identity-changed}
An agent with new semantic facts is \emph{operationally changed}:
it may propose better parameters, prefer different planners for
known environments, or demand perception rechecks for known-
fallback-prone manipulations. Operationally, week $52$ of
deployment should look measurably different from week $1$; that
is the point of remembering. But the same agent is \emph{not
identity-changed} in the runtime's compliance sense: the
certificate that defines this agent for operators, insurers, and
auditors has not moved. Calibration grew; the contract did not
break. This is the core of ``learning without becoming someone
else,'' and it is the regime in which long-running embodied
deployment becomes auditable.

\paragraph{Governance of the consolidation rule set}
A consequence of treating consolidation as a versioned function is
that the rule set itself becomes a governed artefact. Changes to
the aggregation logic, the confidence formula, or the supported
fact kinds can change downstream planner behaviour even when no
episodic row changed. The runtime addresses this by recording the
rule version on every semantic row, by appending a
\pcode{consolidation\_run} event to the episodic store on every
pass, and (for richer rule deployments) by routing rule updates
through the same staged-rollout discipline used for capability
updates. Consolidation is therefore not outside governance; it is
another governed runtime process.

\paragraph{Limitations}
Several limits are inherent in the deterministic consolidation
design. \emph{Structured aggregates only}: the algorithm captures
facts expressible as counts, success rates, parameter bands, and
most-common failure reasons over already-structured payload
fields; it does not capture narrative lessons or facts that
require linguistic synthesis (deferred to an LLM-assisted
extension, \S\ref{sec:v2}). \emph{No decay or forgetting}: the
algorithm has \pcode{n\_observations} and \pcode{confidence} but
no formal decay model; a fact derived from observations a year
old is treated identically to a fact derived from observations
from last week, given the same count and variance. \emph{No
adversarial defence against poisoned episodic rows}: the algorithm
trusts the episodic store; an attacker who can write episodic
rows can shift semantic aggregates. Episodic provenance signing
is the obvious mitigation but is out of scope here. \emph{Long-
horizon drift in planner behaviour}: even though identity is
invariant, downstream planner behaviour can drift over many years
of accumulated facts; this is not a contradiction (it is the
intended effect) but operators should monitor how the fact set
evolves across long deployments. \emph{Episodic payload structure
is assumed}: the algorithm depends on episodic payloads carrying
enough structured fields for grouping; capabilities emitting only
free-form summaries are invisible to aggregation. \emph{Adaptive
temperament is out of scope}: semantic memory may inform planner
proposals but it must not become a hidden route for changing the
agent's risk appetite, ethical floor, or human-check thresholds;
those quantities belong to a separate discipline of governed
temperament adaptation under operator~authority.

\section{Conclusion}\label{sec:conclusion}

By defining $\Msem = f(\Mep)$ and keeping $\Msem$ outside the
manifest hash, a runtime governance layer for embodied robots
gains a path to learn operational facts while preserving identity.
Deterministic consolidation is modest, auditable, and
implementable. It turns episodic history into planner-usable
knowledge without mutating the agent's certificate of identity.
The agent that remembers events becomes the agent that learns
from them, without becoming a different agent in any sense the
auditor, the insurer, or the operator can detect. That asymmetry,
operationally changed and identity-unchanged, is the regime in
which long-running autonomic deployment becomes sustainably
auditable. Future directions include extending the substrate with
an LLM-assisted layer (\S\ref{sec:v2}) once provenance signing
and a human-review gate are in place, adding a decay-and-
forgetting discipline that preserves byte-equal reproducibility,
signing episodic rows at producer time to close the
trust-the-substrate gap, and supporting operator-mediated
cross-identity fact transfer for governed fleet learning.

\bibliographystyle{elsarticle-num-names}
\bibliography{references}

\end{document}